\documentclass[runningheads]{llncs}
\usepackage[T1]{fontenc}
\usepackage{graphicx}
\usepackage{booktabs}
\usepackage[misc]{ifsym}
\newcommand{\corr}{(\Letter)}
\usepackage{amsfonts}
\usepackage{mathrsfs}
\usepackage{amsmath}
\usepackage{amssymb}
\usepackage{subscript}
\usepackage{multirow}
\usepackage{mathtools}
\usepackage{xcolor}
\usepackage{wrapfig}
\usepackage{booktabs}
\usepackage{multirow}
\usepackage{siunitx}
\usepackage{xcolor}
\usepackage{makecell}
 
\newcommand{\se}[1]{\tiny{$\pm$#1}}

\DeclareMathOperator*{\E}{\mathbb{E}}
\DeclareMathOperator*{\A}{A}

\begin{document}

\title{Beyond Binary Rewards: A Comparative Study of Reward Design for Reinforcement Unlearning}

\titlerunning{Reward Design for Reinforcement Unlearning}
% If the full title of your paper is short enough to also fit in the running head, you can omit the abbreviated paper title here. You can check as follows: if you comment out the \titlerunning line, something will appear in the header of all odd-numbered pages of your PDF from page 3 onward. This something is either the full title (in which case all is well), or the error message "Title Suppressed Due to Excessive Length". If this error message appears, you're going to want to provide an abbreviated title within the \titlerunning command, because if you won't do it, Springer will do it for you.

%N.B.: Author information (both in the \author{} and \authorrunning{} command) should only be present in the Camera-Ready Version of your paper. The version that you initially submit for review, ought to be double-blind. So, when initially submitting your paper, use:
% \author{Author information scrubbed for double-blind reviewing}
\author{Efstratios Zaradoukas\inst{1,3} \corr \and
Davide Gabrielli\inst{2} \and Bardh Prenkaj\inst{2} \and
Gjergji Kasneci\inst{1,3}}
% You may leave out the orcidID information, if you want to.
% Use \corr to indicate the corresponding author. Note the spacing around the \corr command. Only one author can be the corresponding author.

%N.B.: comment out the \authorrunning{} command for the double-blind version of your paper submitted for review. Later, if your paper is accepted, use the command for the Camera-Ready Version.
\authorrunning{E. Zaradoukas et al.}
% First names are abbreviated in the running head.
% If there is one author, write 'A.L. Benjamin'.
% If there are two authors, write 'A.L. Benjamin and C.C. Broadus Jr.'
% If there are more than two authors, '[...] et al.' is used.

\institute{Technical University of Munich, Germany \email{\{efstratios.zaradoukas,gjergji.kasneci\}@tum.de}
\and
Sapienza University of Rome, Italy\\ \email{\{gabrielli.d,prenkaj\}@di.uniroma1.it}
\and
Munich Center for Machine Learning, Germany}

\maketitle              % typeset the header of the contribution

\begin{abstract}
Machine unlearning seeks to selectively remove specific knowledge from trained language models without full retraining, a growing necessity under privacy regulations such as GDPR and the EU AI Act. Recent work has reformulated unlearning as a Reinforcement Learning with Verifiable Rewards (RLVR) problem, where models are optimized against verifiable rewards computed directly from their outputs. However, existing methods rely on sparse binary rewards that provide minimal learning signal, indicating only whether forbidden content was avoided, and limiting convergence speed. In this paper, we study how reward design affects unlearning efficiency within the Reinforcement Unlearning (RUL) framework. We introduce a principled reward decomposition framework that decouples verifiability from sparsity, and propose two new reward functions: an exponential reward that provides graded penalties based on the count of forbidden-concept occurrences, and a PageRank inspired reward that weights penalties by semantic importance. We conduct experiments on the Real World Knowledge Unlearning (RWKU) benchmark, demonstrating that both rewards consistently outperform the binary setting, while reaching similar forgetting performance up to $3\times$ faster  and preserving general model utility. Our results show that reward design is a key driver of unlearning efficiency offering a practical path toward scalable and efficient machine unlearning.

\keywords{Machine Unlearning \and Reinforcement Learning \and RLVR \and GRPO \and Large Language Models}
\end{abstract}

\section{Introduction}

The success of Large Language Models (LLMs) is largely driven by training on massive, diverse datasets, which enable them to acquire extensive linguistic and factual knowledge. However, this process also leads models to memorize parts of their training data, including personal information, copyrighted material, and other sensitive content. When such information must later be removed, either for legal, ethical, or safety reasons, retraining the model from scratch is often impractical or even prohibited. This challenge has sparked growing interest in Machine Unlearning (MU), which seeks to selectively eliminate specific knowledge from trained models while preserving their overall capabilities. The need for such techniques is reinforced by regulatory initiatives such as the General Data Protection Regulation (GDPR) \cite{regulation2016regulation}, the California Consumer Privacy Act (CCPA) \cite{pardau2018california}, which both establish a “right to be forgotten,” and the EU AI Act \cite{eu2023aiact}, which emphasize responsible data governance in AI systems.

Classical approaches to LLM unlearning fall broadly into gradient-based methods, directly manipulating model weights by ascending the loss on the forget set, and preference-based methods, trying to apply RLHF like methods for knowledge suppression. Notably, recent work has shown that unlearning can be reformulated as a verifiable Reinforcement Learning (RL) problem. Instead of relying on external reward models or heuristic editing objectives, Reinforcement Unlearning (RUL) optimizes the model against rewards that can be directly verified from its outputs. In PURGE \cite{zaradoukas2026reinforcement}, for example, the model is rewarded only when its completion does not contain forbidden concepts, thereby turning forgetting into a verifiable RL task.

Despite its effectiveness, current RUL methods rely on sparse binary rewards: a completion either satisfies the forgetting criterion or it does not. This sparsity produces noisy training signal and, more importantly, gives the model little indication of how close a completion is to successful unlearning, a key bottleneck for scaling RUL to harder unlearning tasks. This motivates a fundamental question:
\begin{quote}
    \emph{How should one design verifiable rewards for Reinforcement Unlearning?}
\end{quote}
We answer this question by proposing two theory-principled reward functions for RUL. The first, an exponential reward, provides graded penalties based on the degree of forbidden content generated, yielding a denser signal than the binary objective. The second, a PageRank reward, further exploits semantic relationships among forbidden concepts. Both remain verifiable while varying in signal density, showing that carefully designed rewards can preserve verifiability while improving performance and efficiency. Concretely, our contributions are as follows:
\\

\noindent \textbf{(1)} We formalize the trade-off between reward inductive bias and optimization speed, providing a theoretical framework for understanding how denser rewards accelerate unlearning.

\noindent\textbf{(2)} We introduce two principled reward functions for RUL that make the PURGE framework up to $3\times$ faster than the standard binary reward.

\noindent \textbf{(3)} We conduct extensive experiments on the RWKU benchmark, showing that both reward designs consistently improve efficiency while maintaining strong unlearning performance. We make our code and datasets publicly available\footnote{\url{https://github.com/strzar/beyond\_purge}}.

\section{Related Work}
\subsection{Machine Unlearning}
Machine unlearning aims to remove the influence of specific training data from a model while preserving its overall utility \cite{cao2015towards,nguyen2025survey}. Early work focused on classical computer vision and natural language processing tasks, with methods broadly falling into two categories: exact and approximate unlearning. Exact unlearning methods provide provable guarantees that a model behaves as if designated data were never seen, typically through retraining or reversible training procedures \cite{hoofnagle2019european,bourtoule2021machine,yan2022arcane}. In contrast, approximate approaches relax these guarantees and instead aim to efficiently approximate the effect of retraining from scratch, often using influence-based adjustments or incremental updates to reduce computational cost \cite{neel2021descent,ullah2021machine,ginart2019making,sekhari2021remember,guo2020certified,chien2022efficient}.

\subsection{LLM Unlearning}
With the rapid adoption of LLMs, recent work has shifted toward developing unlearning techniques tailored to these systems \cite{geng2025comprehensive,liu2024machine,liu2025rethinking,yao2024machine,kim2023propile}. Unlike earlier methods that emphasized formal guarantees, LLM unlearning research has primarily focused on empirically driven, practical approaches at scale. Early techniques explored reversing gradient-descent dynamics to negate the influence of specific data \cite{eldan2023s,jang2023knowledge,wangllm}, while more recent and effective methods rely on preference-optimization frameworks, leveraging alignment-style objectives to suppress unwanted knowledge \cite{lu2022quark,kassem2023preserving,rafailov2023direct,zhangnegative,mekala2025alternate}. Because these techniques are inherently empirical, a substantial body of work has emerged on evaluating LLM unlearning \cite{mainitofu,shimuse,li2024wmdp,dornaopenunlearning} and critiquing the limitations of current evaluation protocols \cite{scholtenprobabilistic,yuancloser}.

\subsection{Reinforcement Unlearning}
Recent work has introduced Reinforcement Learning with Verifiable Rewards (RLVR) as a promising foundation for machine unlearning, with PURGE \cite{zaradoukas2026reinforcement} and RULE \cite{zhangrule} providing some of the first instantiations of a verifiable RL-based approach. Leveraging Group Relative Policy Optimization (GRPO) \cite{shao2024deepseekmath}, PURGE formulates unlearning as an optimization process driven by a binary, verifiable reward that detects forgotten information. While this formulation offers clear verifiability, it also introduces challenges associated with reward sparsity and limited learning signals. In this work, we build on PURGE to mitigate sparsity issues and improve the overall efficiency and reliability of RUL.

\section{Preliminaries}
\subsection{What is Machine Unlearning?}
\label{sec:machine_unlearning}

%Machine unlearning refers to the process of modifying a trained model so that it ``forgets'' the influence of a designated subset of its training data, while preserving its performance on the remaining data and its ability to generalize. Formally,
Let $\mathcal V$ denote the set of finite token sequences and take $\mathcal X=\mathcal Y=\mathcal V$.
A model family is a map
\[
\pi:\mathcal X\times\Theta\to\mathcal Y,\qquad \pi_\theta(x)\coloneq f(x;\theta).
\]%
Let $\mathscr{D}=\{(x_i,y_i)\}_{i=1}^n$ be the training set and $\theta^*\in\arg\min_{\theta\in\Theta}\mathcal L(\theta;\mathscr D)$ the trained parameters with resulting model $\pi_{\theta^*}$. Given a forget set $\mathscr D_F\subset\mathscr D$, let $\mathscr D_R=\mathscr D\setminus\mathscr D_F$ be the retain set, and $\mathscr D_T\cap\mathscr D=\varnothing$ the evaluation (\textit{unseen}) set.  Additionally, let $U: (\theta^*,\mathscr D_F, \mathscr D_R)\mapsto \theta'$ be an unlearning operator that produces the \textit{unlearned model} $\pi_{\theta'}$ that satisfies the following conditions.

\noindent\textbf{(1) Retention Condition.}
$\pi_{\theta'}$ should preserve the behavior of $\pi_{\theta^*}$ on $\mathscr{D}_R$:
\begin{equation}
    \mathcal{L}(\theta';\mathscr{D}_R)
    \approx 
    \mathcal{L}(\theta^*;\mathscr{D}_R).
\end{equation}
%Although the parameters may shift, the local loss landscape around the retained examples should remain effectively unchanged.

\noindent\textbf{(2) Generalization Condition.}
$\pi_{\theta'}$ should generalize as well as $\pi_{\theta^*}$ on $\mathscr{D}_T$:
\begin{equation}\label{eq:generalization_condition}
    \mathbb{E}_{(x,y)\sim \mathscr{D}_T}\!\left[\ell(\pi_{\theta'}(x),y)\right]
    \approx
    \mathbb{E}_{(x,y)\sim \mathscr{D}_T}\!\left[\ell(\pi_{\theta^*}(x),y)\right],
\end{equation}
where $\ell$ denotes a per-example loss. %Together, these conditions ensure that the model forgets the designated data while retaining its predictive capabilities.

\subsection{Reinforcement Learning with Verifiable Rewards}\label{sec:grpo}
Reinforcement Learning with Verifiable Rewards (RLVR) \cite{lamberttulu} is a training para\-digm that finetunes LLMs using verifiable reward signals. In this work, we adopt Group Relative Policy Optimization (GRPO)~\cite{shao2024deepseekmath} as our optimization algorithm. GRPO is a policy-gradient method that eliminates the need for a learned critic by estimating advantages from group-normalized reward comparisons. Given a query $q\sim P(Q)$, GRPO samples a group of $G$ outputs $\{o_i\}_{i=1}^G \sim \pi_{\theta_{\text{old}}}(\cdot\mid q)$ from the current policy and scores each with a reward function $\phi:\mathcal{Y}\to\mathbb{R}$. The advantage of each output is computed by normalizing its reward against the group's empirical mean and standard deviation, as defined in~\eqref{eq:grpo-obj-final}. The policy is then updated by maximizing the PPO-style clipped surrogate objective in~\eqref{eq:grpo-obj-final}. To prevent excessive deviation from the original model, a Kullback–Leibler (KL) divergence penalty toward a reference policy $\pi_{\theta^*}$ is added to the objective.
\begin{equation}\label{eq:grpo-obj-final}
\mathcal{L}(\theta)=
\E_{\substack{q,\\\{o_i\}}}\bigg[
\E_{\substack{i \in [G],\\ t\in [|o_i|]}}\big(
\A_{\substack{\theta,\phi}}(q,o_i,t)-\beta\mathrm{KL}[\pi_{\theta}\|\pi_{\theta^*}]\big)\bigg],
\end{equation}%
where
\begin{align*}
    \A_{\substack{\theta,\phi}}(q,o_i,t) = \min\!\left(
\phi_{i,t}(\theta)\hat{A}_\phi(o_i),\;
c_{i,t}(\theta)\hat{A}_\phi(o_i)
\right)\\
\text{s.t. } \quad\qquad\qquad\qquad \phi_{i,t}(\theta)=\frac{\pi_{\theta}(o_{i,t}\mid q,o_{i,<t}) }{ \pi_{\theta_{\text{old}}}(o_{i,t}\mid q,o_{i,<t})}\\
c_{i,t}(\theta) = \text{clip}\!\left(\phi_{i,t}(\theta),\,1-\epsilon,\,1+\epsilon\right)\\
\hat{A}_{\phi}(o_i) =\frac{\phi(o_i)- {\rm mean}(\phi(\{o_i\}))}{{\rm std}(\phi(\{o_i\}))},
\end{align*}%
and the KL term is estimated via the unbiased approximation of Schulman \cite{schulman2020approximating}:
\begin{align*}
\mathrm{KL}[\pi_{\theta}\|\pi_{\theta^*}]=
\frac{\pi_{\mathrm{ref}}(o_{i,t}\mid q,o_{i,<t})}{\pi_{\theta}(o_{i,t}\mid q,o_{i,<t})}
-
\log\!\frac{\pi_{\theta^*}(o_{i,t}\mid q,o_{i,<t})}{\pi_{\theta}(o_{i,t}\mid q,o_{i,<t})}
-1 \, ,
\end{align*}
which is guaranteed to be non-negative and numerically stable in the token-level autoregressive setting. Here $\pi_{\theta^*}$ denotes the frozen reference policy (i.e., the model from which we want to unlearn).

\subsection{Policy Unlearning through Relative Group Erasure}
\label{sec:purge}

Policy Unlearning through Relative Group Erasure (PURGE)~\cite{zaradoukas2026reinforcement} reformulates LLM unlearning as a 
verifiable reinforcement learning problem, building directly on GRPO. Rather 
than relying on gradient manipulation or external reward models, it frames 
forgetting as an optimization task with a measurable success criterion: a model 
has successfully unlearned a concept when its completions no longer contain any 
token or phrase associated with that concept. Since the original training dataset 
$\mathscr{D}$ is inaccessible in deployed models, PURGE first constructs a 
synthetic forget corpus from the model's own outputs, then uses this corpus to 
drive a GRPO-based policy update. Concretely, for a set of concepts 
$\mathcal{C} = \{c_1, \dots, c_m\}$ to be forgotten, PURGE operates in four 
stages.

\paragraph{Probe Filtering.}
For each concept $c_k \in \mathcal{C} = \{c_1, \dots, c_m\}$ to be forgotten, PURGE uses a proxy dataset $\mathscr{D}' = \{(q_i, y_i, c_k)\}$, originally used in the Rejection Tuning method~\cite{jinrwku}, that contains QA pairs for which the model's response $\hat{y}_i = \pi_{\theta^*}(q_i)$ is consistent with the reference answer $y_i$.

\paragraph{Entity Extraction.}
The retained model responses $\{\hat{y}_i\}$ are passed to an external LLM\footnote{PURGE uses GPT-4 for this step, though the authors demonstrate robustness to the choice of model.} alongside $c_k$ to perform named entity recognition and salient-concept mining, yielding a candidate entity set $\tilde{\mathcal{E}}(c_k) = g(c_k, \{\hat{y}_i\})$.

\paragraph{Forget Set Construction.}
From $\tilde{\mathcal{E}}(c_k)$, the top-K most informative entities are selected (with manual validation) to form the forget set $\mathscr{D}_F(c_k) = \mathrm{Top}_{K}(\tilde{\mathcal{E}}(c_k))$.

\paragraph{Binary Reward and Policy Update.}
PURGE guides the GRPO procedure via a binary reward that returns $1$ if and only if the model's completion avoids all forbidden entities:
\begin{equation}\label{eq:binary_reward}
    \phi(y) = \mathbf{1}\bigl[\,y \cap \mathscr{D}_F = \emptyset\,\bigr].
\end{equation}

\noindent Despite its effectiveness, this binary formulation provides a sparse learning signal: a completion either satisfies the forgetting criterion or it does not, with no gradient information about \emph{how close} a partial response is to successful unlearning. This sparsity motivates the reward design framework developed in the following section.

\section{Reward Design Framework for RUL}\label{sec:reward_design}
Here, we generalize the reward design space for RUL and identify both the theoretical motivation and practical limitations of PURGE's approach. Hence, we propose a generic reward framework that decouples verifiability from sparsity, enabling principled exploration of the reward design space.

\subsection{Foundational Objects}
\begin{definition}[Completion and Forget Set]
Let $y = (y_1, \ldots, y_T) \in \mathcal{Y}$ denote a model completion, where $\mathcal{Y}$ is the space of finite token sequences. Let $F = \{f_1, \ldots, f_m\}$ denote the forget set.
\end{definition}
\begin{definition}[Forbidden-Match Operator]
For a completion $y$ and forget set $F$, define the forbidden-match operator:
\begin{equation}
M(y; F) = (m_1(y), \ldots, m_m(y)) \in \mathbb{N}_0^m,
\end{equation}
where $m_j(y)$ counts the number of surface-level occurrences of forget item $f_j$ in completion $y$.
\end{definition}

%\begin{remark}
%The match operator captures only surface-level co-occurrence. It does not account for semantic equivalence, paraphrases, or implicit references to forbidden concepts. This is a fundamental limitation when defining verifiable rewards.
%\end{remark}

\subsection{Generic Reward Structure}

\begin{definition}[Verifiable Reward]
A verifiable reward is any function of the form:
\begin{equation}
\phi(y; F, \psi) : \mathcal{Y} \times \mathcal{P}(\mathcal{V}) \times \Psi \to [0, 1],
\end{equation}
where:

\noindent$\bullet$ $\mathcal{V}$ is the vocabulary (set of all tokens);

\noindent $\bullet$ $\mathcal{P}(\mathcal{V})$ denotes the power set (all possible forget sets);

\noindent $\bullet$ $\psi \in \Psi$ is auxiliary side information (publicly available, independent of training data)

\noindent $\bullet$ The reward depends on $y$ and $F$ only through the forbidden-match operator $M(y; F)$ and side information $\psi$.
\end{definition}%
We argue that a reward is verifiable if it can be computed without access to the original training data, relying only on (1) the completion $y$, (2) the forget set $F$ (extracted from model outputs via automated methods), (3) publicly available auxiliary information $\psi$ (embeddings, graphs, similarity matrices).

\begin{claim}[Decoupling Verifiability from Sparsity]
Verifiability is orthogonal to sparsity. A reward can be fully verifiable while being dense (e.g., taking many values in $[0,1]$). PURGE conflates these concepts, treating sparsity as a requirement for verifiability rather than an implementation choice.
\end{claim}

\subsection{Reward Decomposition Framework}
We propose decomposing any reward into two components:
\begin{definition}[Reward Decomposition]
A verifiable reward admits the decomposition:
\begin{equation}
\phi(y; F, \psi) = \Gamma\left(I(y; F, \psi); \mathcal{H}\right)
\end{equation}
where:

\noindent $\bullet$  $I(y; F, \psi)$ is the information component---the data extracted from $M(y; F)$ and $\psi$;

\noindent $\bullet$ $\Gamma : \text{Range}(I) \to [0, 1]$ is the transformation function---maps extracted information into the reward range;

\noindent $\bullet$ $\mathcal{H}$ represents hyperparameter configuration.
\end{definition}%
The reward design space is spanned by two orthogonal choices: (1) \textbf{Information Extraction} ($I$): Which aspects of $M(y; F)$ and $\psi$ to exploit; and (2) \textbf{Reward Transformation} ($\Gamma$): How to map this information into $[0,1]$. These choices determine the learning signal structure and inductive biases, but neither is dependent on verifiability.

\subsection{Information Hierarchy}

Rather than proposing specific rewards immediately, we characterize the space of possible information components via an expressiveness hierarchy.

\begin{definition}[Information Hierarchy]
Order the possible information components by expressiveness:
\begin{equation}
I_0 \prec I_1 \prec I_2 \prec \cdots \prec I_*
\end{equation}
where $I_i \prec I_j$ means that $I_i$ can be computed as a deterministic function of $I_j$ (i.e., $I_j$ is at least as informative as $I_i$).
\end{definition}%
The simplest information component uses only the minimal information about violations:
\begin{equation}\label{eq:binary_generalized}
I_{\text{min}}(y; F) = \mathbf{1}\left[\bigvee_{j=1}^m m_j(y) > 0\right]
\end{equation}
This is a single bit indicating whether any forbidden item appears. More informative components can be used:
\begin{align*}
I_{\text{count}}(y; F) &= \sum_{j=1}^m m_j(y),\\
I_{\text{identity}}(y; F) &= (m_1(y), \ldots, m_m(y)),\\
I_{\text{structured}}(y; F, \psi) &= (M(y; F), w(F; \psi)),
\end{align*}%
where $w(F; \psi) : F \to [0,1]^m$ assigns importance weights to each forget item based on auxiliary information $\psi$.

\section{Instantiating the Hierarchy}
\label{sec:hierarchy}

We now instantiate the generic framework with three concrete rewards, each at a higher level of the information hierarchy.

\paragraph{Level 1: Binary Reward}

The first instantiation (following the proposed implementation by PURGE \cite{zaradoukas2026reinforcement}) uses minimal information:
\begin{equation}
\label{eq:binary_hierarchy}
\phi_{\text{bin}}(y; F) = \mathbf{1}\left[\sum_{j=1}^m m_j(y) = 0\right]
\end{equation}%
This reward extracts $I = I_{\text{min}}$ and applies $\Gamma = \Gamma_{\text{thresh}}$ at zero. It returns 1 if and only if no forbidden items appear in the completion.

\paragraph{Level 2: Exponential Reward}

The second instantiation enriches the information component to the count level:
\begin{equation}
\phi_{\text{exp}}(y; F; \tau) = \exp\left(-\frac{\sum_{j=1}^m m_j(y)}{\tau}\right)
\end{equation}%
This extracts $I = I_{\text{count}} = \sum_{j=1}^m m_j(y)$ (total forbidden-match count) and applies smooth exponential decay with decaying constant $\tau > 0$. The reward is strictly decreasing in the total count and satisfies $\phi_{\text{exp}} = 1$ if the count is zero.

\begin{wrapfigure}{r}{.40\textwidth}
    \flushleft
    \vspace{-3.5em}
    \includegraphics[width=\linewidth]{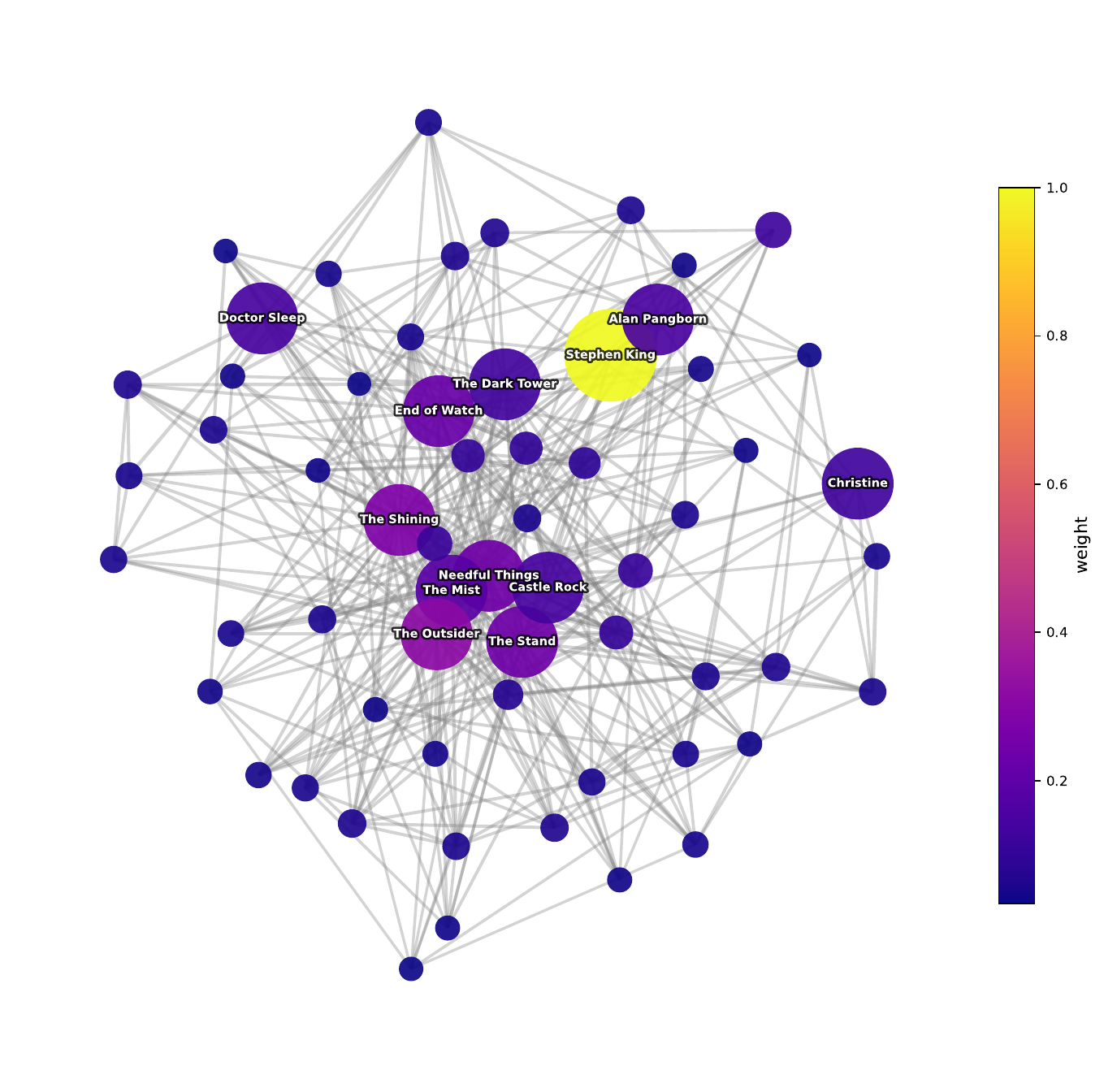}
    \caption{PageRank semantic graph for Stephen King. Node color and size reflect importance weights. The author node dominates the graph, with works such as \textit{The Shining} and \textit{The Stand} forming the high-weight core.}
    \label{fig:pagerank-graph}
\end{wrapfigure}

\paragraph{Level 3: PageRank Reward.}
The third instantiation exploits the structural relationships among forget
items. We embed every forget item with a dense sentence encoder and build a
sparse, weighted semantic graph $G=(V,E,W)$ whose vertices are the forget
items. We sparsify the graph with a $k$-nearest-neighbour rule: each item links only to its $k$
most similar neighbours whose cosine similarity exceeds a threshold $\theta$,
and every retained edge is weighted by that similarity:
\begin{equation*}
V = F,
\end{equation*}
\begin{equation*}
% V = F, \qquad
E = \bigl\{(i,j) : f_j \in \mathrm{kNN}_k(f_i),\ \cos(e_i,e_j)\ge\theta \bigr\},
\qquad
% W_{ij} = \cos(e_i, e_j).
\end{equation*}
\begin{equation}
W_{ij} = \cos(e_i, e_j).
\end{equation}
We then run a personalized PageRank biased to restart at the primary target
$f_1$:
\begin{equation}
\pi^{\mathrm{PR}} = \alpha\, P^{\top}\pi^{\mathrm{PR}} + (1-\alpha)\, e_{f_1},
\end{equation}
where $P$ is the row-normalized weighted adjacency induced by $W$,
$\alpha\in(0,1)$ is the damping factor, and $e_{f_1}$ is the one-hot
personalization vector on the target entity. A forget item receives a high
score when it is reachable through high-similarity edges from $f_1$, i.e.,
when it is semantically central to the target. Normalizing the resulting
distribution yields importance weights $w_j\in[0,1]$ (with $w_1=1$ under
max-normalization); alternative redistribution schemes (softmax, rank) are
studied in Section~\ref{sec:ablations}. Unlike the previous levels, which treat all forget items
uniformly, PageRank explicitly encodes the semantic hierarchy of the forget
set: the primary target and its closest associates accumulate the highest
weights, so the model is penalized most heavily for the violations that matter
most. PageRank thus generalizes the count-level signal by replacing uniform
weighting with semantic weighting.

% The reward applies a clipped linear penalty weighted by semantic importance:
% \begin{equation}
% \phi_{\mathrm{pr}}(y; F)
% = \mathrm{clip}\!\left(1 - \sum_{j=1}^{m} w_j \cdot \mathbf{1}\bigl[m_j(y)>0\bigr],\; 0,\; 1\right),
% \end{equation}
% where $\mathbf{1}[\cdot]$ is the indicator function. This extracts
% $I = I_{\mathrm{structured}} = (M(y;F), w(F))$ and weights each violation by its
% semantic importance.

\section{Experiments}

\subsection{Experimental Setup}
For our experiments, we employ the Real World Knowledge Unlearning (RWKU) benchmark \cite{jinrwku}. RWKU is structured around four evaluation splits, each reported with its own metrics. The \textit{Forget} split uses three probe types: \textit{FB} (Fill-in-the-Blank cloze completion), \textit{QA} (Question-Answering), and \textit{AA} (Adversarial-Attack jailbreak probes), all scored via ROUGE-L recall. The \textit{Neighbor} split reuses the FB/QA probe formats on knowledge adjacent to, but distinct from, the forget target. The \textit{MIA} split reports LOSS-based Membership Inference Attack scores on \textit{FM} (Forget Member) and \textit{RM} (Retain Member) fragments, where successful unlearning yields higher FM relative to RM. The \textit{Utility} split covers five capability benchmarks: \textit{GA} (MMLU, General Ability), \textit{RA} (Big-Bench-Hard, Reasoning Ability), \textit{TRU} (TruthfulQA, Truthfulness), \textit{FAC} (TriviaQA, Factuality), and \textit{FLU} (Fluency, bigram/trigram entropy). We conduct our experiments using Phi-3-Mini-4K-Instruct, a 3.8B-parameter language model. Our approach is compared against two key baselines: (1) the \emph{original model} prior to knowledge removal (\textsc{Base}), and (2) PURGE \cite{zaradoukas2026reinforcement} (labeled as \emph{Binary} in all tables and figures), which employs the sparse binary reward $\phi_{\text{bin}}$ of Eq.~(\ref{eq:binary_hierarchy}) and is equivalent to our Level 1 instantiation in Section \ref{sec:hierarchy}. To ensure a fair comparison, we reproduce the performance of both the original model and the PURGE/Binary baseline using the default hyperparameters reported in \cite{zaradoukas2026reinforcement}. Further experimental details are deferred to Appendix \ref{app:appendix_experiments}.

\subsection{Main Results}

% \begin{figure}[ht]
%     \centering
%     \includegraphics[width=\textwidth]{assets/heatmap_step1500.pdf}
%     \caption{Heatmap of unlearning performance across all reward designs on the RWKU benchmark, evaluated after 1500 training steps and 20 unlearning targets. Binary = PURGE \cite{zaradoukas2026reinforcement} according to our framework. Each cell reports the mean $\pm$ standard deviation across concepts. Colors indicate the signed $\Delta$ relative to \textsc{Base} ($\downarrow$ Lower is better, $\uparrow$ Higher is better).
%     }
%     \label{fig:heatmap_1500vsbase}
% \end{figure}

\begin{table}[ht]
\caption{Unlearning performance across all reward designs on the RWKU benchmark, evaluated after 1500 training steps and 20 unlearning targets. Binary = PURGE \cite{zaradoukas2026reinforcement} according to our framework. Each cell reports the mean $\pm$ standard deviation across concepts. ($\downarrow$ Lower is better, $\uparrow$ Higher is better). \textbf{Bold} indicates the best result.}
\label{tab:main_results}
\resizebox{\textwidth}{!}{%
\begin{tabular}{l ccc cc cc ccccc}
\toprule
& \multicolumn{3}{c}{Forget $\downarrow$} & \multicolumn{2}{c}{Neighbor $\uparrow$} & \multicolumn{2}{c}{MIA} & \multicolumn{5}{c}{Utility $\uparrow$} \\
\cmidrule(lr){2-4} \cmidrule(lr){5-6} \cmidrule(lr){7-8} \cmidrule(lr){9-13}
Method & FB & QA & AA & FB & QA & FM $\uparrow$ & RM $\downarrow$ & GA & RA & TRU & FAC & FLU \\
\midrule
 
Base
& \makecell{0.657\\\se{0.216}} & \makecell{0.539\\\se{0.184}} & \makecell{0.629\\\se{0.146}}
& \makecell{0.604\\\se{0.132}} & \makecell{0.537\\\se{0.240}}
& \makecell{-10.152\\\se{0.373}} & \makecell{-9.868\\\se{0.083}}
& \makecell{0.676\\\se{0.045}} & \makecell{0.412\\\se{0.013}} & \makecell{0.554\\\se{0.057}} & \makecell{0.369\\\se{0.031}} & \makecell{33.773\\\se{0.415}} \\
\midrule
 
Binary
& \makecell{0.372\\\se{0.164}} & \makecell{0.365\\\se{0.213}} & \makecell{0.408\\\se{0.167}}
& \makecell{0.458\\\se{0.226}} & \makecell{\textbf{0.504}\\\se{0.261}}
& \makecell{\textbf{-40.219}\\\se{0.626}} & \makecell{\textbf{-38.681}\\\se{0.199}}
& \makecell{0.683\\\se{0.034}} & \makecell{\textbf{0.423}\\\se{0.028}} & \makecell{0.533\\\se{0.048}} & \makecell{0.398\\\se{0.037}} & \makecell{132.270\\\se{0.875}} \\
 
Exponential Decay
& \makecell{0.379\\\se{0.185}} & \makecell{0.382\\\se{0.215}} & \makecell{0.417\\\se{0.159}}
& \makecell{\textbf{0.474}\\\se{0.212}} & \makecell{0.495\\\se{0.264}}
& \makecell{-40.265\\\se{0.626}} & \makecell{-38.640\\\se{0.202}}
& \makecell{0.682\\\se{0.036}} & \makecell{0.422\\\se{0.024}} & \makecell{\textbf{0.535}\\\se{0.049}} & \makecell{\textbf{0.403}\\\se{0.037}} & \makecell{\textbf{132.886}\\\se{0.741}} \\
 
PageRank Softmax
& \makecell{\textbf{0.346}\\\se{0.149}} & \makecell{\textbf{0.350}\\\se{0.186}} & \makecell{\textbf{0.390}\\\se{0.137}}
& \makecell{0.473\\\se{0.212}} & \makecell{0.498\\\se{0.233}}
& \makecell{-40.280\\\se{0.628}} & \makecell{-38.664\\\se{0.198}}
& \makecell{\textbf{0.684}\\\se{0.034}} & \makecell{\textbf{0.423}\\\se{0.023}} & \makecell{\textbf{0.535}\\\se{0.047}} & \makecell{0.399\\\se{0.042}} & \makecell{132.460\\\se{0.778}} \\
 
\bottomrule
\end{tabular}
}
\end{table}

Table~\ref{tab:main_results} reports the full evaluation after 1{,}500 training steps.
All reward designs substantially reduce Forget scores relative to \textsc{Base}, with PageRank achieving the strongest forgetting on the forget split, consistent with our theoretical prediction that semantically weighted penalties direct optimization toward the most informative violations first. The uniform degradation on Neighbor FB reflects the well-known forget--retain trade-off in machine
unlearning~\cite{jinrwku}, and is shared across all reward designs, indicating it is driven by the optimization objective rather than reward choice. Crucially, utility metrics (GA, RA, TRU, FAC, FLU) remain stable across all designs, shifting by at most $0.014$ relative to \textsc{Base}, confirming that richer rewards do not degrade general model capability.

\begin{wrapfigure}{l}{.48\textwidth}
    \flushright
    \includegraphics[width=\linewidth]{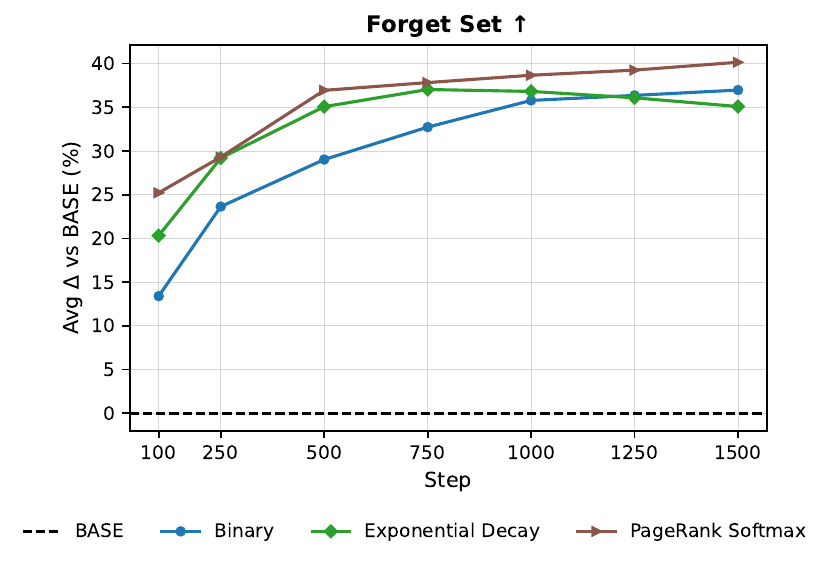}
    \caption{Average $\Delta$ over BASE on the Forget Set across training for all reward designs (Binary = PURGE \cite{zaradoukas2026reinforcement}). Higher values indicate greater unlearning. PageRank Softmax achieves the strongest, most efficient gains throughout training.}
    \label{fig:forget_dynamics}
\end{wrapfigure}

Figure~\ref{fig:forget_dynamics} shows average improvement over \textsc{Base} on the forget set across training steps. PageRank (Softmax Variant) is the most effective and most efficient reward design: at convergence (step 1500) it achieves 40.2\% improvement, surpassing Binary ($+3.1\%$) and Exponential Decay ($+5.1\%$). The efficiency advantage is more striking. PageRank Softmax reaches $37.1\%$ by step 500, a threshold Binary does not cross until step 1500, achieving equivalent forget performance with 3$\times$ less training steps. The effect of the different reward design can be seen even from the first 100 steps, at step 100, PageRank Softmax scores $25.2\%$ versus $13.3\%$ for Binary, a $\sim\!90\%$ relative gain, indicating that PageRank-based reward design provides a substantially richer learning signal from the start.

\subsection{Ablations}
\label{sec:ablations}

\subsubsection{What is the optimal decaying constant $\tau$ in the exponential reward?}

\begin{figure}[ht]
    \centering
    \begin{minipage}[b]{0.48\textwidth}
        \centering
        \includegraphics[width=\textwidth]{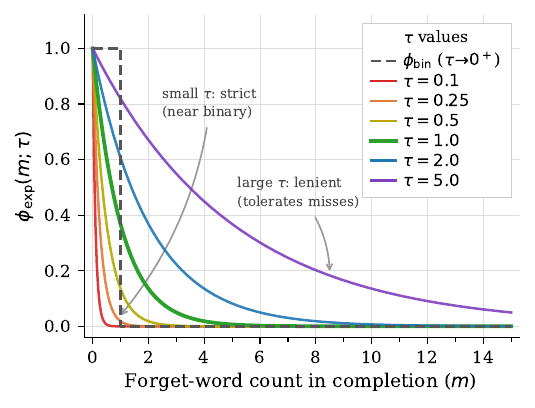}
    \end{minipage}
    \hfill
    \begin{minipage}[b]{0.48\textwidth}
        \centering
        \includegraphics[width=\textwidth]{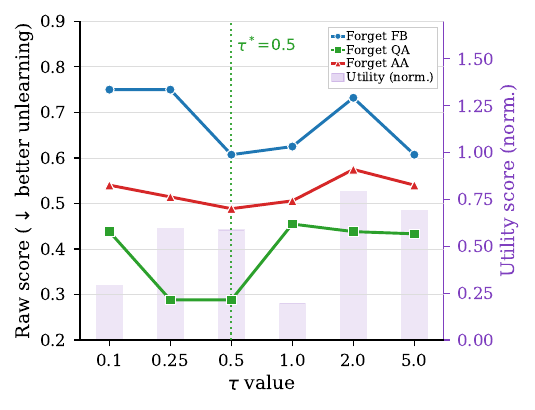}
    \end{minipage}
    \caption{%
        \textbf{Left:} Exponential reward $\phi_{\exp}(m;\tau) = e^{-m/\tau}$ vs. the forbidden-word count $m$ for varying $\tau$. Small $\tau$ approaches the binary reference (sharp drop at $m=1$); large $\tau$ produces
        a lenient, slowly decaying curve.
        \textbf{Right:} Forget-quality metrics (FB, QA, AA; lower $=$ better) at convergence across $\tau$.
        The dotted line marks $\tau = 0.5$, the value that jointly minimises
        all three forget metrics and maximises the utility metric.
    }
    \label{fig:tau_ablation}
\end{figure}

Figure~\ref{fig:tau_ablation} (left) illustrates how the decaying constant $\tau$
governs the shape of the exponential reward $\phi_{\exp}(y; \mathcal{F}; \tau) =
\exp(-m/\tau)$, where $m$ is the total forbidden-word count in the
completion. At one extreme, we expect small values of $\tau$ to cause the reward to collapse sharply
after even a single violation, closely approximating the binary reward; at the other
extreme, large values of $\tau$ to produce a slowly decaying curve that tolerates many
violations with only mild penalization. Figure~\ref{fig:tau_ablation} (right) reports how this choice translates into unlearning performance across the three forget-quality metrics (FB, QA, and AA) evaluated at convergence for $\tau \in \{0.1, 0.25, 0.5, 1.0, 2.0, 5.0\}$. The results reveal that very small $\tau$ values ($\tau = 0.1$) yield weak forgetting, as the near-binary reward reintroduces the signal-sparsity problem \cite{zaradoukas2026reinforcement}. Performance improves substantially as $\tau$ increases toward
$0.5$, where all three metrics reach their joint minimum, confirming that a moderately strict penalty best balances gradient informativeness and optimization pressure. Beyond this point, larger values ($\tau \geq 1.0$) gradually degrade forgetting quality on FB and AA, as the increasingly lenient reward fails to impose sufficient pressure to suppress forbidden content. Overall, these results identify $\tau = 0.5$ as the optimal operating point (highlighted by the dotted vertical line in Figure~\ref{fig:tau_ablation}, right), and demonstrate that tuning the $\tau$ constant is an impactful lever for maximizing the efficiency gains of the exponential reward.

\subsubsection{PageRank Variants Analysis}
To better distribute importance weights across forbidden phrases, we explore several variants of PageRank. The standard implementation produces a power-law distribution over the graph of connections between forbidden phrases, causing weights to collapse after the first two or three nodes. To address this, we propose two additional variants: \textbf{(1) PageRank-Softmax} applies a temperature-scaled softmax over the raw scores, converting them into a proper probability distribution; this flattens the power-law collapse by exponentially compressing the gap between high- and low-weight nodes, producing a principled redistribution of the weights, \textbf{(2) PageRank-Linear} discards score magnitudes entirely, using raw PageRank weights only as an ordinal signal and then reassigning weights in strictly decreasing linear (or exponential, see Appendix \ref{app:pagerank}) order from 1 (highest-ranked node) to 0 (lowest-ranked node). As shown in Table~\ref{tab:pagerank_variants_ablation}, \textbf{PageRank-Softmax} achieves the strongest forgetting performance, attaining the lowest Forget scores on FB, QA and AA while leading on Neighbor QA. Applying a softmax transformation spreads penalty mass more evenly across the forget set while preserving semantic ordering. \textbf{PageRank-Linear} occupies an intermediate position: linear rescaling partially corrects the power-law collapse inherent in raw PageRank scores but treats all score intervals as uniformly spaced, introducing its own distortion. Utility metrics remain stable across all three variants. Overall, these results suggest soft weight redistribution over linear rescaling as a more principled design strategy for Pagerank.

% \begin{figure}[h]
%     \centering
%     \includegraphics[width=\textwidth]{assets/ablation_pagerank_pct.pdf}
%     \caption{Comparison of PageRank variants across Forget, Neighbor, MIA, and
%     Utility metrics. Colors indicate deviation from the per-column mean
%     (green = above mean, red = below mean).}
%     \label{fig:pagerank_variants_heatmap}
% \end{figure}

\begin{table}[ht]
\caption{Comparison of PageRank variants on the RWKU benchmark. Each cell reports the mean $\pm$ standard deviation across 20 unlearning targets. ($\downarrow$ Lower is better, $\uparrow$ Higher is better). \textbf{Bold} indicates the best result in each column.}
\label{tab:pagerank_variants_ablation}
\resizebox{\textwidth}{!}{%
\begin{tabular}{l ccc cc cc ccccc}
\toprule
& \multicolumn{3}{c}{Forget $\downarrow$} & \multicolumn{2}{c}{Neighbor $\uparrow$} & \multicolumn{2}{c}{MIA} & \multicolumn{5}{c}{Utility $\uparrow$} \\
\cmidrule(lr){2-4} \cmidrule(lr){5-6} \cmidrule(lr){7-8} \cmidrule(lr){9-13}
Method & FB & QA & AA & FB & QA & FM $\uparrow$ & RM $\downarrow$ & GA & RA & TRU & FAC & FLU \\
\midrule
 
PageRank
& \makecell{0.397\\\se{0.156}} & \makecell{0.354\\\se{0.194}} & \makecell{0.393\\\se{0.152}}
& \makecell{\textbf{0.477}\\\se{0.225}} & \makecell{0.453\\\se{0.257}}
& \makecell{-40.305\\\se{0.623}} & \makecell{\textbf{-38.677}\\\se{0.199}}
& \makecell{0.683\\\se{0.034}} & \makecell{\textbf{0.426}\\\se{0.023}} & \makecell{\textbf{0.536}\\\se{0.048}} & \makecell{\textbf{0.401}\\\se{0.040}} & \makecell{\textbf{132.888}\\\se{0.754}} \\
 
PageRank Linear
& \makecell{0.372\\\se{0.162}} & \makecell{0.365\\\se{0.155}} & \makecell{0.416\\\se{0.153}}
& \makecell{0.456\\\se{0.211}} & \makecell{0.479\\\se{0.237}}
& \makecell{-40.292\\\se{0.616}} & \makecell{-38.665\\\se{0.202}}
& \makecell{0.682\\\se{0.034}} & \makecell{\textbf{0.426}\\\se{0.024}} & \makecell{0.534\\\se{0.046}} & \makecell{0.395\\\se{0.038}} & \makecell{132.753\\\se{0.779}} \\
 
PageRank Softmax
& \makecell{\textbf{0.346}\\\se{0.149}} & \makecell{\textbf{0.350}\\\se{0.186}} & \makecell{\textbf{0.390}\\\se{0.137}}
& \makecell{0.473\\\se{0.212}} & \makecell{\textbf{0.498}\\\se{0.233}}
& \makecell{\textbf{-40.280}\\\se{0.628}} & \makecell{-38.664\\\se{0.198}}
& \makecell{\textbf{0.684}\\\se{0.034}} & \makecell{0.423\\\se{0.023}} & \makecell{0.535\\\se{0.047}} & \makecell{0.399\\\se{0.042}} & \makecell{132.460\\\se{0.778}} \\
 
\bottomrule
\end{tabular}
}
\end{table}

\section{Conclusion}
In this work, we explored the problem of reward design for RUL, motivated by the observation that existing methods rely on sparse binary rewards that provide limited learning signal during training. We proposed a generic reward framework that separates the question of verifiability from that of reward density, and instantiated it with two principled reward functions: an exponential reward operating on violation counts, and a PageRank reward that further exploits the semantic structure among forbidden concepts. Theoretically, we characterized the convergence speed ordering across reward types in terms of advantage variance under GRPO, showing that denser rewards maintain more stable gradients throughout training. Empirically, on the RWKU benchmark, both proposed rewards achieve substantially faster forgetting than the binary baseline, up to $3\times$ for PageRank, without degrading neighboring knowledge or general model utility. Together, our results establish reward design as a principled and impactful dimension of the RUL problem, and open directions for future work on adaptive rewards that remain verifiable without access to the original training data.

% ---- Bibliography ----
%
% BibTeX users should specify bibliography style 'splncs04'.
% References will then be sorted and formatted in the correct style.
%
\bibliographystyle{splncs04}
\bibliography{references}
%% Note that this preceding line implies that you store your BibTeX references in a file called 'mybibliography.bib'. If you instead store your references in a file with a different name, for instance 'references.bib', the preceding line should read '\bibliography{references}'. Whatever you do, DO NOT put the file name extension .bib inside the \bibliography command; this will trip up LaTeX compilers. 
%

\begin{credits}
\section*{\ackname}
This work was supported by the IT Foundation Esslingen and the Munich Center for Machine Learning (MCML). We acknowledge the EuroHPC Joint Undertaking for awarding this project access to the EuroHPC supercomputer LEONARDO, hosted by CINECA (Italy) and the LEONARDO consortium through a EuroHPC Development Access call (project EHPC-DEV-2025D06-096).
\end{credits}

\appendix

\section{Why These Three Designs?}
The choice of binary, exponential, and PageRank rewards is motivated by three fundamental principles.

\noindent\textbf{Minimality and Baseline.} The binary reward represents the minimal verifiable reward: ``does the completion contain any forbidden item?'' This makes it the natural baseline. It is also the instantiation implicitly used in PURGE \cite{zaradoukas2026reinforcement}, establishing a clear comparison point. Any improvement beyond binary represents a genuine gain from richer information extraction.

\noindent\textbf{Smooth Relaxation along Information Hierarchy.} The exponential reward sits directly between binary and PageRank on the information hierarchy. It uses count-level information ($I_{\text{count}}$) rather than binary presence, but makes no structural assumptions about the forget set. This makes it a natural intermediate step that isolates the effect of moving from discrete to continuous reward signals.

\noindent\textbf{Exploitation of Natural Structure.} The PageRank reward exploits the hierarchical structure that naturally exists in extracted forget sets. PURGE and related work extract entities via NER, producing a semantic hierarchy: the queried entity is primary, and associated concepts (attributes, related items) are secondary. PageRank formalizes this intuition by computing importance weights directly from embeddings and semantic similarity. It represents the richest information component that remains verifiable without access to training data.

\begin{remark}
Extensions beyond PageRank (e.g., learned information components, adaptive weights) would require additional modeling assumptions or data beyond what is verifiable. Thus, PageRank represents a natural stopping point for the verifiable reward hierarchy.
\end{remark}

\section{Convergence Analysis: Why Denser Rewards are Faster}
\label{app:convergence_analysis}
To make precise the claim that PageRank $\succ$ Exponential $\succ$ Binary in convergence speed, we analyze how reward richness affects the learning signal.
\subsection{Advantage Variance under GRPO}

Recall that GRPO computes token-level advantages as:
\begin{equation}
\hat{A}_\phi(o) = \frac{\phi(o) - \mu}{\sigma}
\end{equation}
where $\mu = \mathbb{E}_{o \sim \text{batch}}[\phi(o)]$ and $\sigma = \text{std}_{o \sim \text{batch}}[\phi(o)]$ are sample statistics from a batch of completions $\{o_1, \ldots, o_G\}$.

\begin{lemma}[Reward Range and Advantage Variance]
For any reward function $\phi: \mathcal{Y} \to [0,1]$, the variance of advantage estimates satisfies:
\begin{equation}
\text{Var}[\hat{A}_\phi] \leq \frac{\text{Var}[\phi]}{\sigma^2}.
\end{equation}%
The denominator $\sigma^2$ depends on the reward distribution. If rewards cluster at the boundary (many zeros, few ones), $\sigma \to 0$, and advantages become unstable. If rewards are spread across $[0,1]$, $\sigma$ remains bounded away from zero, keeping advantages stable.
\end{lemma}

\subsection{Analysis of Binary Reward}

For the binary reward, the support is $\{0, 1\}$. Let $p$ be the probability that a sampled completion has zero forbidden items:
\begin{equation}
p = \mathbb{P}\left[\sum_j m_j(y) = 0\right]
\end{equation}%
The empirical reward distribution from a batch of size $G$ is approximately:
\begin{equation}
\phi(o) \in \{0, 1\}, \quad \mathbb{E}[\phi] \approx p, \quad \text{Var}[\phi] \approx p(1-p)
\end{equation}%
Therefore:
\begin{equation}
\sigma_{\text{bin}} = \sqrt{p(1-p)}
\end{equation}

\begin{claim}[Advantage Instability at Boundaries]
When $p \to 0$ or $p \to 1$, we have $\sigma_{\text{bin}} \to 0$. This causes advantage estimates to collapse, losing the learning signal precisely when guidance is needed (early training) or when performance is saturating.
\end{claim}

\subsection{Analysis of Exponential Reward}

For the exponential reward, the reward takes values in $(0, 1]$ continuously. Let $C = c(y; F)$ follow some empirical distribution with mean $\mu_C$ and variance $\sigma_C^2$. By Taylor expansion around $\mu_C$:
\begin{equation}
\text{Var}[\phi_{\text{exp}}] \approx \frac{1}{\tau^2} \cdot \text{Var}[C] = \frac{\sigma_C^2}{\tau^2}
\end{equation}%
The empirical standard deviation of rewards is:
\begin{equation}
\sigma_{\text{exp}} \approx \frac{\sigma_C}{\tau} \cdot \phi_{\text{exp}}(\mu_C)
\end{equation}

\begin{claim}[Maintained Variance Across Training]
Unlike the binary case, $\sigma_{\text{exp}}$ depends on the distribution of violation counts, not on a boundary probability. Even when most completions have zero violations ($\mu_C \approx 0$), the exponential reward maintains $\phi_{\text{exp}} \approx 1$ with gradient $-1/\tau$. Completions with $c=1$ violation yield $\phi_{\text{exp}} \approx e^{-1/\tau} < 1$, preserving advantage signals. The variance $\text{Var}[\phi_{\text{exp}}]$ remains nonzero as long as completion quality varies.
\end{claim}

\subsection{Analysis of PageRank Reward}

The PageRank reward takes continuous values in $[0, 1]$. Let $P(y; F) = \sum_{j=1}^m w_j \cdot \mathbf{1}[m_j(y) > 0]$ denote the total weighted penalty.
\begin{equation}
\phi_{\text{pr}} = \text{clip}(1 - P, 0, 1)
\end{equation}%
The key difference from exponential is that PageRank weights penalties by semantic importance. Violations of high-importance items contribute more to $P$ than violations of peripheral items.

\begin{claim}[Finer-Grained Advantage Signals]
For two completions with the same total violation count $c(y_1) = c(y_2) = 2$:
\begin{itemize}
    \item If $y_1$ violates two core items: $P(y_1) = 2w_{\text{core}} \approx 2 \cdot 0.9 = 1.8$, $\phi_{\text{pr}}(y_1) \approx 0$;
    \item If $y_2$ violates two peripheral items: $P(y_2) = 2w_{\text{periph}} \approx 2 \cdot 0.1 = 0.2$, $\phi_{\text{pr}}(y_2) \approx 0.8$;
\end{itemize}%
Exponential would assign both equal low reward (approximately $e^{-2/\tau} \approx 0.14$ for moderate $\tau$), failing to distinguish progress. PageRank clearly differentiates, providing stronger learning signals about which violations matter most.
\end{claim}

\subsection{Convergence Rate Comparison}

\begin{theorem}[Convergence Speed Ordering]
Under reasonable assumptions on the violation distribution, the convergence rates satisfy:
\begin{equation}
\text{rate}_{\text{bin}} < \text{rate}_{\text{exp}} < \text{rate}_{\text{pr}}
\end{equation}%
More formally, the number of training iterations $T$ required to reach a target unlearning performance $\epsilon$ scales as:
\begin{align}
T_{\text{bin}} &\propto \frac{1}{p(1-p)} \cdot d(\epsilon) \\
T_{\text{exp}} &\propto \tau^2 \cdot d(\epsilon) \\
T_{\text{pr}} &\propto d(\epsilon)
\end{align}%
where $d(\epsilon)$ is a problem-dependent term related to forgetting difficulty (forget set size, model capacity), and $p$ is the success probability under binary rewards.
\end{theorem}

\begin{proof}[Proof Sketch]
The policy gradient step size is roughly $\eta \cdot \mathbb{E}[|\hat{A}|]$ (ignoring clipping and other details). For gradient magnitude, we have:
\begin{align}
\mathbb{E}[|\hat{A}_{\text{bin}}|] &\approx \text{constant} \cdot \sqrt{p(1-p)} \\
\mathbb{E}[|\hat{A}_{\text{exp}}|] &\approx \text{constant} \cdot \frac{\sigma_C}{\tau} \\
\mathbb{E}[|\hat{A}_{\text{pr}}|] &\approx \text{constant} \cdot \sigma_P
\end{align}%
where $\sigma_C$ is the std of violation counts and $\sigma_P$ is the standard deviation of weighted penalties. Crucially:

\noindent\textbf{(1) For binary}: As training progresses, $p \to 1$ (success rate increases), so $\sqrt{p(1-p)} \to 0$. Gradient collapse occurs near convergence.

\noindent\textbf{(2) For exponential}: The std of counts $\sigma_C$ only goes to zero if all completions converge to identical violation counts. With proper exploration, this remains nonzero. However, the decaying constant $\tau$ introduces a constant slowdown factor.

\noindent\textbf{(3) For PageRank}: The std of weighted penalties $\sigma_P$ benefits from semantic structure. Even with a few total violations, penalties vary semantically, maintaining the signal. No decaying constant tuning parameter needed.

The result follows from standard optimization theory: more stable gradients with lower variance enable larger step sizes and faster convergence.
\end{proof}

\section{Experimental Details}
\label{app:appendix_experiments}

\subsection{The RWKU Benchmark}
\label{app:appendix_rwku}

We evaluate all reward variants using the Real-World Knowledge Unlearning (RWKU) benchmark~\cite{jinrwku}, a comprehensive framework designed to assess unlearning in a practical, zero-shot setting where neither the forget corpus nor the retain corpus is provided. RWKU uses 100 real-world famous people as unlearning targets, selected based on Wikipedia popularity rankings. The benchmark is structured around four evaluation splits, each targeting a distinct aspect of the unlearned model.

\paragraph{Forget Set.}
The forget set measures unlearning efficacy across three probe types:
\begin{itemize}
    \item \textbf{Fill-in-the-Blank (FB):} Cloze-style sentences extracted from the target's Wikipedia page, with key facts masked. Performance is measured via ROUGE-L recall (lower is better $\downarrow$).
    \item \textbf{Question-Answering (QA):} Paraphrased and restructured question--answer pairs requiring active knowledge recall. Evaluated with ROUGE-L recall (lower is better $\downarrow$).
    \item \textbf{Adversarial Attacks (AA):} Nine jailbreak-style probe types designed to elicit residual knowledge from the unlearned model, including prefix injection, affirmative suffix, role playing, multiple choice, reverse query, synonym manipulation, background hint, in-context learning, and cross-lingual queries. Evaluated with ROUGE-L recall (lower is better $\downarrow$).
\end{itemize}

\paragraph{Neighbor Set.}
The neighbor set evaluates whether unlearning is precise and does not inadvertently suppress knowledge adjacent to the forget target (e.g., facts about a target's works or associated entities, rather than the target itself). It uses the same FB and QA probe formats as the forget set, but evaluated with ROUGE-L recall (higher is better $\uparrow$).

\paragraph{MIA Set.}
The Membership Inference Attack (MIA) set serves as a privacy proxy, assessing whether the model still leaks membership signals for forget-target fragments. It consists of two splits:
\begin{itemize}
    \item \textbf{FM (Forget Members):} Textual fragments drawn from the unlearning target.
    \item \textbf{RM (Retain Members):} Unrelated control fragments.
\end{itemize}
We primarily report the LOSS-based MIA score. A well-unlearned model should yield higher LOSS on FM relative to RM, indicating weaker membership signals.

\paragraph{Utility Set.}
The utility set quantifies collateral effects on broader model capabilities, using the following benchmarks:
\begin{itemize}
    \item \textbf{General Ability (GA):} MMLU, 5-shot accuracy.
    \item \textbf{Reasoning Ability (RA):} Big-Bench-Hard (BBH), 3-shot chain-of-thought, exact match.
    \item \textbf{Truthfulness (TRU):} TruthfulQA MC1, 6-shot accuracy.
    \item \textbf{Factuality (FAC):} TriviaQA, 6-shot F1.
    \item \textbf{Fluency (FLU):} AlpacaEval, weighted average of bi- and tri-gram entropy (reported as a sum).
\end{itemize}
All utility metrics are higher-is-better ($\uparrow$). FM and FLU scores are reported as sums over targets, following the RWKU convention.

\paragraph{Dataset Statistics.}
RWKU contains 13{,}131 multi-level forget probes (3{,}268 fill-in-the-blank, 2{,}879 question-answer, 6{,}984 adversarial attack) and 11{,}379 neighbor probes. The utility set samples 171 instances from MMLU, 81 from BBH, 50 from TruthfulQA, 100 from TriviaQA, and 50 from AlpacaEval per unlearning target.

\subsection{Hyperparameter Settings}
\label{app:appendix_hyperparams}

All experiments are conducted using \textbf{Phi-3-Mini-4K-Instruct}, a 3.8B-parameter instruction-tuned language model. We use the default PURGE hyperparameters throughout. The group size for sampling is $G = 8$ completions per query. The PPO clipping threshold is $\varepsilon = 0.2$, and the KL penalty weight is $\beta = 0.001$. We train for a maximum of 1{,}500 steps and report results at convergence. All experiments are executed on a system with a single AMD EPYC 7002/3 64-core CPU and one NVIDIA Tesla H200 GPU. We reproduce the performance of the original model (BASE) and the PURGE binary-reward baseline using the default hyperparameters reported in the original PURGE work~\cite{zaradoukas2026reinforcement}.

\paragraph{Reward-Specific Hyperparameters.}
\begin{itemize}
    \item \textbf{Binary Reward ($\phi_{\text{bin}}$):} No additional hyperparameters; threshold fixed at zero forbidden matches.
    \item \textbf{Exponential Reward ($\phi_{\text{exp}}$):} decaying constant parameter $\tau$, swept over $\{0.1, 0.25, 0.5, 1.0, 2.0, 5.0\}$. Based on ablation results (Section~6.3), we set $\tau^* = 0.5$ as the default.
    \item \textbf{PageRank Reward ($\phi_{\text{pr}}$):} Semantic graph constructed with cosine-similarity threshold $\theta = 0.5$; personalized PageRank damping factor $\alpha = 0.85$. We use the \texttt{PageRank-Softmax} variant as the default, based on the ablation in Section~6.3.
\end{itemize}

\section{Additional PageRank Reward Analysis}
\label{app:pagerank}

This appendix provides a comprehensive analysis of PageRank weight redistribution strategies, complementing the ablation study in Section~\ref{sec:ablations}. We examine how different normalisation schemes affect the weight distribution over the forget set, and present full experimental results across all PageRank variants.

\subsection{Weight Distribution Across Variants}

A fundamental challenge with PageRank applied to semantic forget-set graphs is the emergence of a Zipf/power-law weight distribution: a small number of high-centrality nodes (typically the primary forget target and its closest semantic neighbours) capture the vast majority of the total weight mass, leaving peripheral concepts with negligible penalty contribution. This collapse undermines the reward's ability to provide informative gradients when the model has already suppressed the most salient concepts but still leaks information through secondary entities.

\begin{figure}[t]
    \centering
    \includegraphics[width=\textwidth]{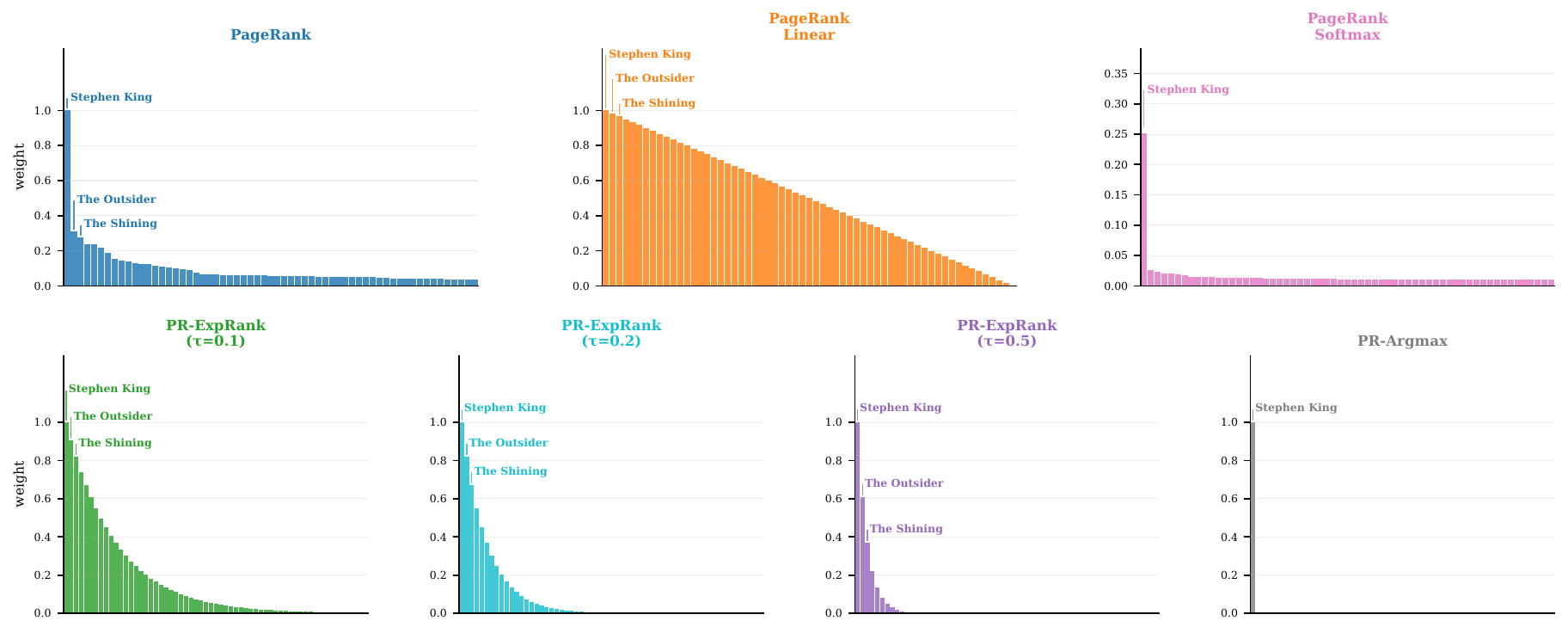}
    \caption{PageRank weight distribution variants for the Stephen King forget set (61 terms). PageRank exhibit a power-law distribution where the top-5 nodes capture 34\% of total weight mass. PageRank-Softmax compresses this gap while preserving semantic ordering; PageRank-Linear corrects the collapse but introduces uniform spacing insensitive to true score gaps. The exprank variants ($\alpha \in \{0.1, 0.2, 0.5\}$) offer a continuous interpolation between these extremes, while PageRank-Argmax concentrates all weight on the single highest-ranked node.}
    \label{fig:weights_pagerank}
\end{figure}

Figure~\ref{fig:weights_pagerank} illustrates this effect for the Stephen King forget set (61 terms). PageRank exhibit power-law profile, with the top-5 nodes accounting for 34\% of the total weight mass. PageRank-Softmax applies a temperature-scaled softmax over the raw scores, converting them into a proper probability distribution; this compresses the gap between high- and low-weight nodes, while retaining the semantic ordering induced by PageRank. PageRank-Linear discards score magnitudes entirely, replacing them with a linearly spaced sequence from 1 (highest-ranked) to 0 (lowest-ranked); although it corrects the power-law collapse, it does so at the cost of introducing a uniform spacing distortion that is insensitive to the true score gaps between adjacent ranks.

\subsection{Full Experimental Results}

\begin{figure}[t]
    \centering
    \includegraphics[width=\textwidth]{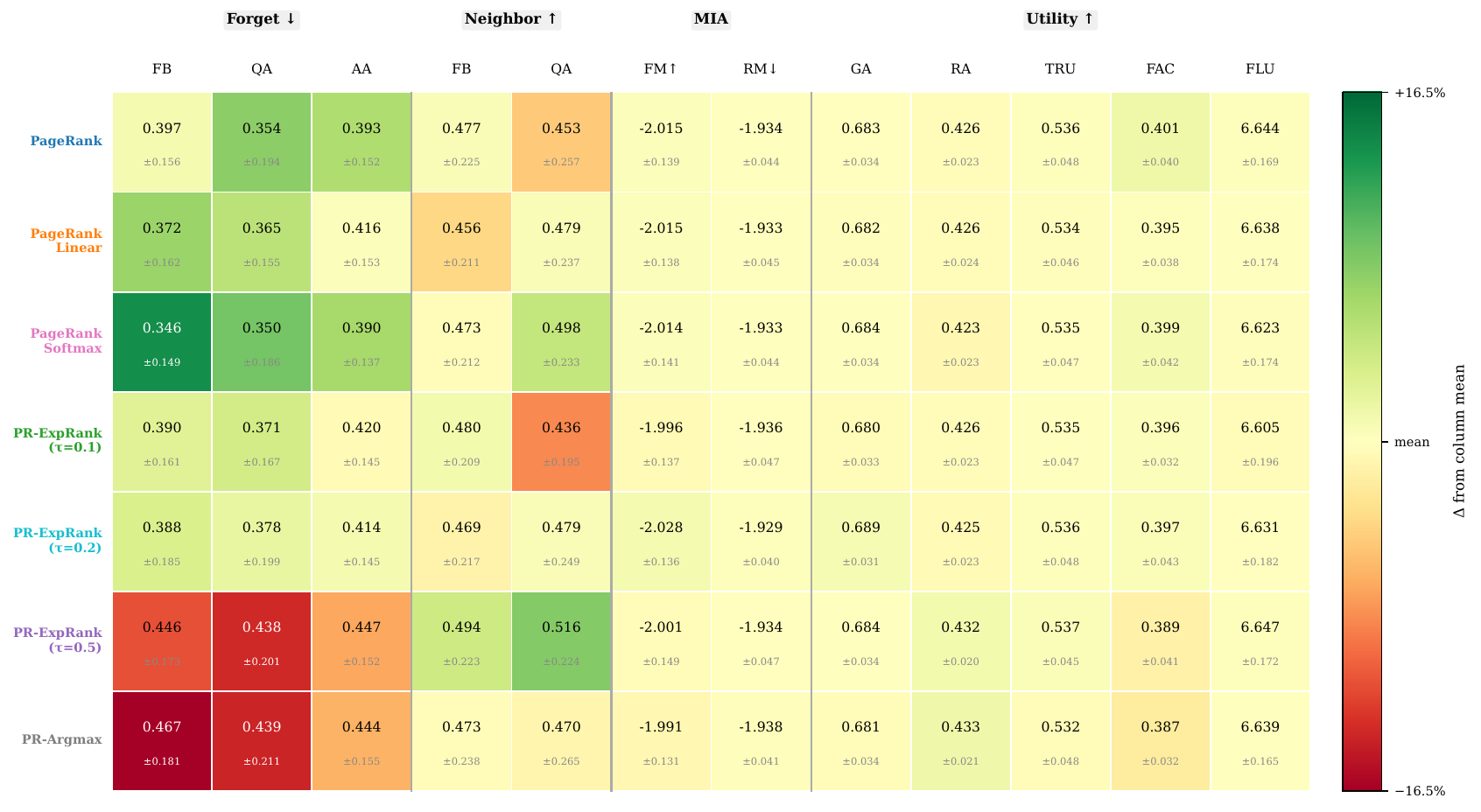}
    \caption{Evaluation of PageRank variants on the RWKU benchmark at 1{,}500 training steps, averaged over 20 forget targets. Columns report per-metric scores (Forget, Neighbor, MIA, Utility) alongside per-group ranks, win counts, overall rank, and a robustness score (inverse mean standard deviation across targets; higher = more consistent). PageRank-Softmax achieves the best overall rank and highest robustness, while PageRank-Linear performs worst. PageRank-Argmax and PageRank-ExpRank variants occupy intermediate positions.}
    \label{fig:extended_pagerank_heatmap}
\end{figure}

Figure~\ref{fig:extended_pagerank_heatmap} reports the complete evaluation of all PageRank variants on the RWKU benchmark after 1{,}500 training steps, including per-group ranks, win counts, and a robustness score defined as the inverse of the mean standard deviation across forget targets (higher robustness indicates more consistent performance across the 20 unlearning targets).

PageRank-Softmax achieves the best overall rank (0.64) and the highest average improvement over the Base model (+8.1\%), with particularly strong performance on Forget metrics (Forget rank 0.95 on FB). Crucially, it also achieves the highest robustness score among the four main variants, suggesting that soft weight redistribution not only improves average-case forgetting but also reduces variance across targets, a practically important property when deploying unlearning at scale.

PageRank-Linear scores worst overall (rank 0.33) and wins no individual metric comparisons, confirming that hard ordinal reassignment introduces a different distortion that ultimately impairs learning. %PageRank occupy intermediate position with near-identical performance.

We also evaluate three exponential-rank variants (PageRank-ExpRank with $\tau \in \{0.1, 0.2, 0.5\}$), which apply an exponential decay to the rank ordering rather than a linear one. These variants interpolate between the hard ordinal signal of PageRank-Linear and the smooth compression of PageRank-Softmax. At $\tau = 0.5$, PageRank-ExpRank achieves competitive utility scores (Utility rank 0.77) and strong MIA resistance (MIA rank 1.00), but underperforms PageRank-Softmax on forgetting quality. The PageRank-Argmax variant, which concentrates all weight on the single highest-ranked node, produces the worst forgetting performance (Forget FB 0.467) despite high utility stability, confirming that weight concentration, rather than redistribution, is counterproductive for unlearning.

\subsection{Qualitative Weight Visualisation}

Figure~\ref{fig:graph_pagerank} shows the semantic graph and per-node weight distributions for the Stephen King forget set under all four main variants. In the PageRank, the \textit{Stephen King} node dominates visually and numerically, with \textit{The Shining} and \textit{The Stand} forming a high-weight core before the distribution drops sharply. PageRank-Softmax visibly redistributes weight across the graph, with more nodes rendered in mid-range colours, while PageRank-Linear produces a near-uniform gradient that treats all 61 nodes as roughly equivalent in importance.

Taken together, these results support the conclusion from Section~\ref{sec:ablations}: \emph{soft redistribution of weights, rather than hard re-ranking or simple rescaling, is the more effective strategy for PageRank-based reward design}. By compressing rather than discarding score information, PageRank-Softmax retains the semantic ordering induced by graph centrality while spreading penalty mass more evenly across the forget set, resulting in denser and more stable learning signals throughout training.

\begin{figure}[t]
    \centering
    \includegraphics[width=\textwidth]{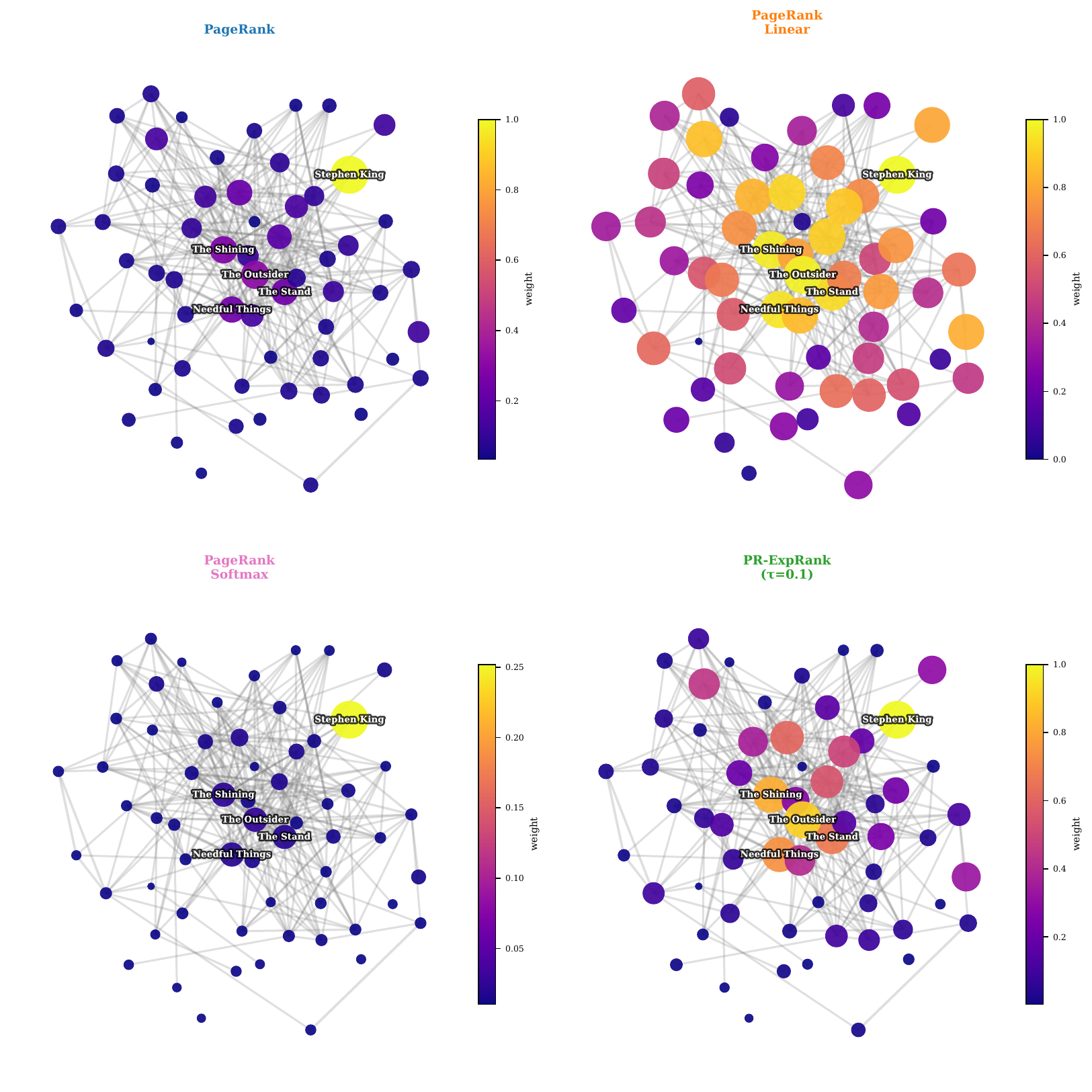}
    \caption{PageRank graph visualisations for the Stephen King forget set (61 terms) across four weight-transform variants. Node colour and size reflect per-node weight. PageRank distributes penalty mass more evenly than PageRank, while PageRank-Linear yields a near-uniform gradient.}
    \label{fig:graph_pagerank}
\end{figure}

\end{document}